\documentclass[letterpaper, 10 pt, conference]{ieeeconf}
\IEEEoverridecommandlockouts    
\usepackage{cite}
\usepackage{amsmath,amssymb,mathrsfs,amsfonts}
\usepackage{bbm}
\usepackage{algorithmic}
\usepackage{graphicx}
\usepackage{subcaption}
\usepackage{color}

\newtheorem{problem}{Problem}

\newtheorem{theorem}{Theorem}

\newtheorem{remark}{Remark}

\begin{document}

\title{\bf Task Space Tracking of Soft Manipulators: Inner-Outer Loop Control Based on Cosserat-Rod Models}
\author{Tongjia Zheng, Qing Han, and Hai Lin
\thanks{*This work was supported by the National Science Foundation under Grant No. CNS-1830335, IIS-2007949.}
\thanks{Tongia Zheng and Hai Lin are with the Department of Electrical Engineering, University of Notre Dame, Notre Dame, IN 46556, USA. (tzheng1@nd.edu, hlin1@nd.edu.)}
\thanks{Qing Han is with the Department of Mathematics, University of Notre Dame, Notre Dame, IN 46556, USA. (Qing.Han.7@nd.edu.)}
}

\maketitle

\thispagestyle{empty}
\pagestyle{empty}

\begin{abstract}
Soft robots are robotic systems made of deformable materials and exhibit unique flexibility that can be exploited for complex environments and tasks.
However, their control problem has been considered a challenging subject because they are of infinite degrees of freedom and highly under-actuated.
Existing studies have mainly relied on simplified and approximated finite-dimensional models.
In this work, we exploit infinite-dimensional nonlinear control for soft robots.
We adopt the Cosserat-rod theory and employ nonlinear partial differential equations (PDEs) to model the kinematics and dynamics of soft manipulators, including their translational motions (for shear and elongation) and rotational motions (for bending and torsion).
The objective is to achieve position tracking of the whole manipulator in a planar task space by controlling the moments (generated by actuators).
The control design is inspired by the energy decay property of damped wave equations and has an inner-outer loop structure. 
In the outer loop, we design desired rotational motions that rotate the translational component into a direction that asymptotically dissipates the energy associated with position tracking errors. 
In the inner loop, we design inputs for the rotational components to track their desired motions, again by dissipating the rotational energy.
We prove that the closed-loop system is exponentially stable and evaluate its performance through simulations.
\end{abstract}


\section{Introduction}
Soft robots are artificial bodies made of continuously deformable and compliant materials \cite{rus2015design}.
Compared with conventional rigid-body robots, the compliant structure endows soft robots with unique advantages such as being inherently safe when interacting with humans and being able to adapt to constrained and crowded environments.
As a result, soft robots have found many applications including medical surgeries and interventions \cite{burgner2015continuum} and underwater maneuvers \cite{marchese2014autonomous}.

Despite the structural advantages, the control problem of soft robots has been considered a challenging subject \cite{della2021model} for two main reasons.
First, due to the continuous deformability, soft robots have infinite degrees of freedom and are inherently infinite-dimensional nonlinear systems, yet the inputs are always finite-dimensional because we can only equip finitely many independent actuators on the robots. 
Second, dimensionality aside, the number of input variables is usually much less than the number of state variables.
(Here a ``variable'' is a continuous function.)
In one word, soft robots are highly under-actuated systems.

The existing effort has mainly focused on developing finite dimensional approximations of soft robots' kinematics and dynamics, such as those that are based on the piecewise constant curvature (PCC) assumption \cite{webster2010design} or finite element methods (FEM) \cite{polygerinos2015modeling}.
PCC models suggest ignoring the linear strains (like shear) that are sometimes negligible compared with the angular strains (like bending) and further assume that a soft manipulator consists of a finite number of curved segments with constant curvatures.
In this way, the configuration space is approximated by a reduced number of finite-dimensional variables.
This approach has been widely adopted and produced fruitful results ranging from kinematic control to dynamic control in the last two decades \cite{marchese2016dynamics, della2020model, franco2021energy}.
However, this over-simplification suffers from low accuracy and might also produce local singularities, especially in the presence of significant body and external loads.
FEM is a numerical method for solving partial differential equations, which represents the deformable shape as a very large set of mesh nodes together with the information of their neighbors \cite{duriez2013control, goury2018fast}.
While FEM is a powerful tool for simulating deformations of various geometric shapes, it significantly relies on linearization and reduction for control purposes.

Cosserat-rod models, also known as geometrically exact models, are infinite-dimensional models for soft manipulators which are based on continuum mechanics and are considered more accurate \cite{simo1988dynamics, antman2005nonlinear, rucker2011statics, renda2014dynamic}.
They describe the kinematics and dynamics of a soft manipulator using a system of two coupled nonlinear partial differential equations (PDEs), one for the translational/linear deformations and the other for the rotational/angular deformations.
The PCC and FEM models, to some extent, may be considered as finite-dimensional approximations of the Cosserat-rod models \cite{della2021model}.
Moreover, since they are mechanics-based, the role of actuators can be systematically formulated into Cosserat-rod models \cite{till2019real, renda2020geometric, janabi2021cosserat}.
Despite modeling accuracy, these PDEs are nonlinear and highly under-actuated, which are very difficult to control due to the lack of a well-developed control theory for infinite-dimensional nonlinear systems.
As a result, the existing effort of the control design based on Cosserat-rod models has mainly relied on discretization \cite{grazioso2019geometrically, george2020first, doroudchi2021configuration}, or assuming a full-actuation \cite{chang2020energy, zheng2022pde}.

In this work, we design feedback controllers directly based on an under-actuated Cosserat-rod model without approximations.
The control objective is to achieve position tracking of the whole manipulator in a planar task space by designing the internal moments (generated by actuators) which are treated as the input variables.
We recognize that the complete system has a lower-triangular structure in the sense that the rotational motion can be viewed as a virtual input of the translational component.
Therefore, we adopt an inner-outer loop design and exploit the energy decay property of damped wave equations for each loop.
In the outer loop, we design desired rotational motions that rotate the translational motions into a direction that converts the tracking error system into a damped wave equation whose energy is known to decay exponentially.
In the inner loop, we design inputs for the rotational component to track their desired motions, again by converting the rotational error into a damped wave equation.
We prove that this inner-outer loop feedback controller achieves the exponential stability of position tracking in task space.
Simulations are included to validate the performance of the proposed controller.

The rest of the paper is organized as follows.
We introduce the Cosserat-rod model and the control problem in Section \ref{section:modeling}. 
The inner-outer loop control design is presented in Section \ref{section:control}.
In Section \ref{section:simulation}, simulations are conducted to validate the algorithms.
Section \ref{section:conclusion} summarizes the contribution.

\section{Modeling and Problem Statement}
\label{section:modeling}
The special orthogonal group is defined by $SO(3)=\{R\in\mathbb{R}^{3\times3}\mid R^TR=I,\operatorname{det}R=1\}$.
The associated Lie algebra is given by $\mathfrak{so}(3)=\{A\in\mathbb{R}^{3\times3}\mid A=-A^T\}$.
Define the hat operator $(\cdot)^\wedge:\mathbb{R}^3\to\mathfrak{so}(3)$ by the condition that $u^\wedge v=u\times v$ for all $u,v\in\mathbb{R}^3$, where $\times$ denotes the cross product.
Let $(\cdot)^\vee:\mathfrak{so}(3)\to\mathbb{R}^3$ be its inverse operator, i.e., $(u^\wedge)^\vee=u$.

Cosserat-rod models are geometrically exact models that describe the dynamic response of long and thin deformable rods undergoing external forces and moments \cite{simo1988dynamics, antman2005nonlinear} and have been widely used to model soft manipulators \cite{rucker2011statics, renda2014dynamic, renda2020geometric, till2019real, janabi2021cosserat, grazioso2019geometrically, george2020first, doroudchi2021configuration, chang2020energy, zheng2022pde}.
A Cosserat rod is idealized as a spatial curve and a family of cross-sections (Fig. \ref{fig:Cosserat rod}).
Let $\{e_1,e_2,e_3\}$ be the standard basis of $\mathbb{R}^3$.
Let $s\in[0,\ell]$ be the arc length parameter of the undeformed centerline, where $\ell$ is the undeformed length.
The position of centerline is specified by $p:[0,\ell]\times[0,T]\to\mathbb{R}^3$.
The rotation of each cross-section is specified by $R:[0,\ell]\times[0,T]\to SO(3)$.
The columns $\{b_1,b_2,b_3\}$ of $R$ can be seen as a body-attached basis.
We thus have two  types of coordinate frames: one is the fixed global frame $\{e_1,e_2,e_3\}$, and the other is a family of body-attached local frames $\{b_1,b_2,b_3\}$.
The deformation of the rod is uniquely decomposed into the linear strains $q$ (for shear and elongation) and the angular strains $u$ (for bending and torsion), which are defined in the local frames (Fig. \ref{fig:strains}).

\begin{figure}[t]
    \centering
    \begin{subfigure}[b]{0.8\columnwidth}
        \centering
        \includegraphics[width=\textwidth]{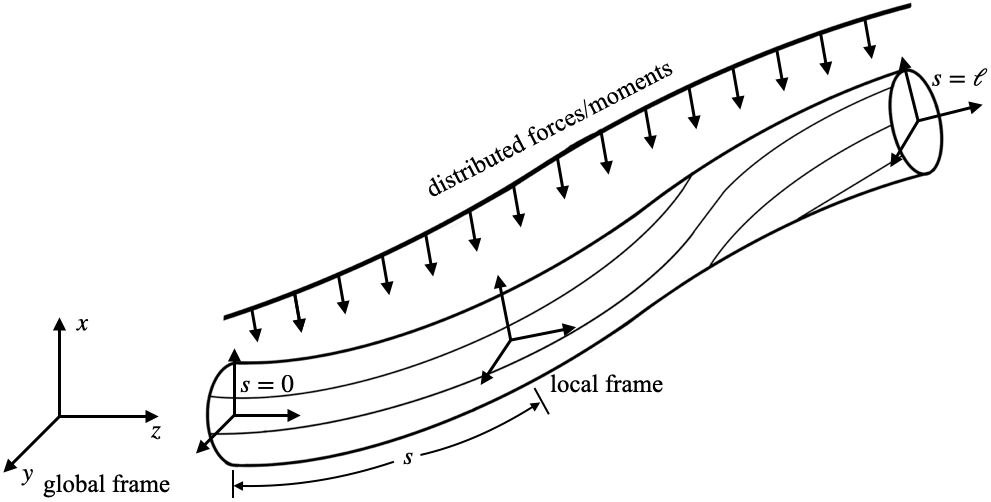}
        \caption{A Cosserat rod.}
        \label{fig:Cosserat rod}
    \end{subfigure}
    
    \begin{subfigure}[b]{0.6\columnwidth}
        \centering
        \includegraphics[width=\textwidth]{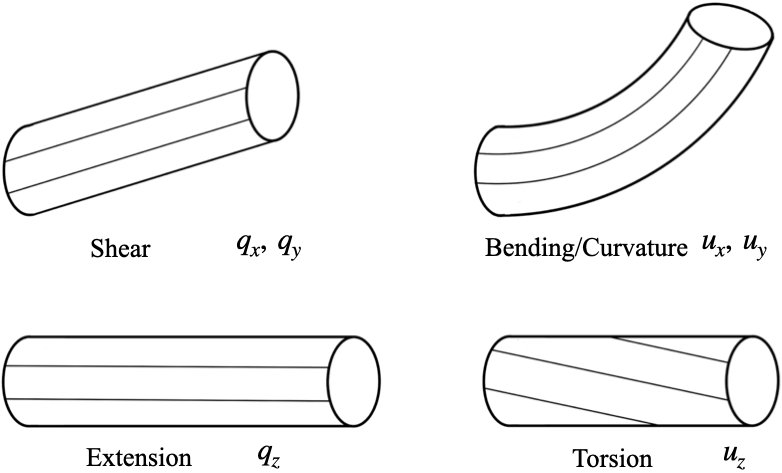}
        \caption{Four types of strains.}
        \label{fig:strains}
    \end{subfigure}
    \caption{Illustration of Cosserat rods.}
\end{figure}

\begin{table}[]
\small
\setlength{\tabcolsep}{4pt}
\renewcommand{\arraystretch}{1.2}
    \centering
    \begin{tabular}{rl}
        $p$ & $:[0,\ell]\times[0,T]\to\mathbb{R}^3$ global position \\
        $R$ & $:[0,\ell]\times[0,T]\to SO(3)$ global rotation matrix \\
        $V,v$ & $:[0,\ell]\times[0,T]\to\mathbb{R}^3$ global/local linear velocities \\
        $w$ & $:[0,\ell]\times[0,T]\to\mathbb{R}^3$ local angular velocity \\
        $q$ & $:[0,\ell]\times[0,T]\to\mathbb{R}^3$ local linear strain \\
        $u$ & $:[0,\ell]\times[0,T]\to\mathbb{R}^3$ local angular strain \\
        $\mathcal{N}$ & $:[0,\ell]\times[0,T]\to\mathbb{R}^3$ global total internal force \\
        $\mathcal{M}$ & $:[0,\ell]\times[0,T]\to\mathbb{R}^3$ global total internal moment \\
        $n$ & $:[0,\ell]\times[0,T]\to\mathbb{R}^3$ local internal elastic force \\
        $m$ & $:[0,\ell]\times[0,T]\to\mathbb{R}^3$ local internal elastic moment \\
        $n_c$ & $:[0,\ell]\times[0,T]\to\mathbb{R}^3$ local internal force input \\
        $m_c$ & $:[0,\ell]\times[0,T]\to\mathbb{R}^3$ local internal moment input \\
        $\mathcal{F}$ & $:[0,\ell]\times[0,T]\to\mathbb{R}^3$ global external force \\
        $\mathcal{L}$ & $:[0,\ell]\times[0,T]\to\mathbb{R}^3$ global external moment \\
        $\rho$ & $\in\mathbb{R}_+$ density \\
        $\sigma$ & $\in\mathbb{R}_+$ cross-sectional area \\
        $J$ & $\in\mathbb{R}^{3\times3}$ rotational inertia matrix \\
        $K_1,K_2$ & $\in\mathbb{R}^{3\times3}$ linear/angular stiffness matrices \\
    \end{tabular}
    \caption{Nomenclature in the 3D case}
    \label{tab:nomenclature}
\end{table}

Using the nomenclature in Table. \ref{tab:nomenclature}, the kinematics and dynamics of a Cosserat rod are characterized 
by the following set of PDEs \cite{simo1988dynamics, antman2005nonlinear}:
\begin{align}
    p_s & =Rq \label{eq:p_s} \\
    R_s & =Ru^\wedge \label{eq:R_s} \\
    p_t & =Rv \label{eq:p_t} \\
    R_t & =Rw^\wedge \label{eq:R_t} \\
    \mathcal{N}_s+\mathcal{F} & =(R\rho\sigma v)_t \label{eq:linear momentum} \\
    \mathcal{M}_s+p_s\times \mathcal{N}+\mathcal{L} & =(R\rho Jw)_t \label{eq:angular momentum}
\end{align}
where $(\cdot)_t:=\frac{\partial}{\partial t}(\cdot)$ and $(\cdot)_s:=\frac{\partial}{\partial s}(\cdot)$ are partial derivatives.
Assuming the rod is initially in a straight configuration on the $z$-axis, the initial condition is given by
\begin{align}\label{eq:IC}
\begin{split}
    & p(s,0)=[0~0~s]^T,\quad p_t(s,0)=0, \\
    & R(s,0)=I,\quad R_t(s,0)=0,
\end{split}
\end{align}
where $I\in\mathbb{R}^{3\times3}$ is the identity matrix.
Assume one end $(s=0)$ is fixed and the other end $(s=\ell)$ is free, the boundary condition is given by
\begin{align}\label{eq:BC 3D}
\begin{split}
    & p(0,t)=0,\quad R(0,t)=I, \\
    & \mathcal{N}(\ell,t)=0, \quad \mathcal{M}(\ell,t)=0.
\end{split}    
\end{align}


Following the nomenclature in \cite{webster2010design}, we refer to $\{p,R\}$ as the task space variables, and $\{q,u\}$ as the configuration space variables.
Equations \eqref{eq:p_s}-\eqref{eq:R_s} are the forward kinematics. Equations \eqref{eq:p_t}-\eqref{eq:R_t} are the kinematic equations.
Equations \eqref{eq:linear momentum}-\eqref{eq:angular momentum} are the dynamic equations for momentum balance.
We assume the control inputs act on the system in the form of internal forces and moments, denoted by $n_c$ and $m_c$ respectively.
This assumption is valid for tendon and fluidic actuators \cite{renda2020geometric}.
Following \cite{renda2020geometric}, the total internal forces and moments are given by
\begin{align}
    \mathcal{N} & =R(n-n_c), \label{eq:total internal force}\\
    \mathcal{M} & =R(m-m_c). \label{eq:total internal moment}
\end{align}
Assuming linear constitutive laws, the internal elastic forces and moments (caused by deformation) are given by
\begin{align}
    n & =K_1(q-\bar{q}) \label{eq:linear constitutive law},\\
    m & =K_2(u-\bar{u}) \label{eq:angular constitutive law},
\end{align}
where $\{\bar{q}(s),\bar{u}(s)\}$ are the undeformed values of $\{q,u\}$, $K_1$ and $K_2$ are positive-definite and diagonal.
In our problem, $\mathcal{F}=\rho\sigma ge_3$ and $\mathcal{L}=0$ where $g$ is the gravity.

Note that there are only four independent state variables (or twelve if recalling that each variable has three elements).
To obtain a minimum representation of the complete system, we choose $\{p,R,V,w\}$ as the state variables.
By substituting \eqref{eq:p_s}-\eqref{eq:R_s} and \eqref{eq:total internal force}-\eqref{eq:angular constitutive law} into \eqref{eq:linear momentum}-\eqref{eq:angular momentum} and using the fact that $V=Rv$, we obtain the following task space representation:
\begin{align}\label{eq:task representation 3D}
\begin{split}
    p_t & =V \\
    \rho\sigma V_t & =\big(R(n-n_c)\big)_s+\rho\sigma ge_3 \\
    R_t & =Rw^\wedge \\
    \rho Jw_t & =R^T\big(R(m-m_c)\big)_s+R^Tp_s^\wedge R(n-n_c) \\
    & \quad -w^\wedge\rho Jw,
\end{split}
\end{align}
where, according to \eqref{eq:p_s}-\eqref{eq:R_s} and \eqref{eq:linear constitutive law}-\eqref{eq:angular constitutive law},
\begin{align*}
    n & =K_1(R^Tp_s-\bar{q}) \\
    m & =K_2((R^TR_s)^\vee-\bar{u}).
\end{align*}
This is a system of two nonlinear hyperbolic equations, one for the translational motions $p$ and the other for the rotational motions $R$.

\begin{remark}
Relevant results concerning the well-posedness of \eqref{eq:task representation 3D} can be found in \cite{rodriguez2020boundary}.
With mild conditions, one can show that $H^1$ solutions of $\{p,R\}$ exist globally, which is essentially due to the energy conservation property of Cosserat rods.
One can also obtain the semi-global existence of $C^2$ solutions with extra smallness assumptions on the initial conditions \cite{li2001semi}.
The complete treatment of the well-posedness issue is beyond the scope of this work and is left for future research.
\end{remark}

In practice, system \eqref{eq:task representation 3D} is under-actuated for two reasons.
First, the inputs $\{n_c,m_c\}$ are the internal forces and moments generated by actuators, which are always finite-dimensional because there are only finitely many independent actuators.
Second, $m_c$ is usually a function of $n_c$, i.e., we cannot control the translational and rotational motions independently.
We use the models for fluidic actuators derived in \cite{till2019real} as an example.
Consider a soft manipulator with multiple hollow actuation chambers offset from the centerline.
Let $\alpha$ be the number of chambers.
Let $r_i(s)$ be the vector from the cross-section center of mass to the center of the $i$-th chamber.
Assume a uniform air pressure $P_i(t)$ in each chamber.
Let $A_i$ be the cross-sectional area of the chamber.
Then the local internal force and moment generated by the $i$-th chamber are
\begin{align*}
    n_i^c & =P_iA_ie_3, \\
    m_i^c & =r_i^\wedge n_i^c=P_iA_ir_i^\wedge e_3.
\end{align*}
The total local internal force and moment are given by
\begin{align}
    n_c & =\sum_{i=1}^\alpha P_iA_ie_3, \\
    m_c & =\sum_{i=1}^\alpha P_iA_ir_i^\wedge e_3. \label{eq:actuator model}
\end{align}
We see that $n_c$ and $m_c$ are essentially simultaneously determined by an $\alpha$-dimensional air pressure vector $P=[P_1~\dots~P_\alpha]^T$.

To simplify the design, we directly treat $m_c$ as the control input.
In the implementation, one needs to find values of $\{P_i\}$ to approximate the $m_c$ that we design.
We will impose $n_c=0$ to simplify the stability proof.
This is not a restriction because as long as the actuators are paired symmetrically on two sides of the manipulator, we can let their forces counteract each other to generate only moments.
For example, let $r_i=-r_j$ (symmetric placement) and $P_i=-P_j$.

In this work, we consider the planar case where the manipulator only moves on the $yz$-plane, assuming the gravity also lies on this plane.
In this case, one element of every translational variable and two elements of every rotational variable are constantly zero.
The corresponding equations can be removed from the complete system \eqref{eq:task representation 3D}.
We will still use the same notations for the reduced vectors and matrices, e.g., $p=[p_2~p_3]^T$ and $w=w_1$ where $(\cdot)_i,i=1,2,3$ is the $i$-th element of the original vector.
Assuming $n_c=0$, the task space representation in the planar case is given by:
\begin{align}\label{eq:task representation 2D}
\begin{split}
    p_{tt} & =(RK_3R^Tp_s-RK_3\bar{q})_s+ge_3, \\
    \theta_{tt} & =K_4\theta_{ss}-\frac{1}{\rho J}(m_c)_s+\hat{p}_sRK_5(R^Tp_s-\bar{q}),
\end{split}
\end{align}
where $K_3=K_1/(\rho\sigma)$, $K_4=K_2/(\rho J)$, $K_5=K_1/(\rho J)$, $\hat{p}=[-p_3~p_2]$, $\theta:[0,\ell]\times[0,T]\to\mathbb{R}$ is the rotation angle about the $x$-axis, and
\begin{align*}
    R=\begin{bmatrix}
        \cos\theta & -\sin\theta \\
        \sin\theta & \cos\theta
    \end{bmatrix}.
\end{align*}
The boundary conditions are correspondingly simplified as:
\begin{align}\label{eq:BC}
\begin{split}
    & p(0,t)=0,\quad \theta(0,t)=0, \\
    & (R^Tp_s)(\ell,t)-\bar{q}=0, \quad (K_2\theta_s-m_c)(\ell,t)=0.
\end{split}    
\end{align}

Assume the objective is to track a position trajectory $p^*(s,t)\in\mathbb{R}^2$ in the task space (alternatively written as $p_*$ when the superscript position is needed for other notations) which is as smooth as needed with uniformly bounded derivatives and satisfies the following boundary conditions:
\begin{align} \label{eq:desired trajectory}
    p^*(0,t)=0,\quad p^*_s(\ell,t)=p_s(\ell,t).
\end{align}
The condition at $s=\ell$ is a mild condition to simplify the stability analysis.
In practice, for a given desired trajectory $p_*(s,t)$, it is easy to regulate its value in a small neighborhood of the free end $s=\ell$ such that the regulated trajectory is almost the same as the original $p_*$ while satisfying \eqref{eq:desired trajectory}.

The control problem is stated below.

\begin{problem}
Consider \eqref{eq:task representation 2D}.
Design $m_c$ such that $p(\cdot,t)\to p^*(\cdot,t)$ as $t\to\infty$.
\end{problem}

\section{Inner-outer Loop Control Design}
\label{section:control}
In this section, we design an infinite-dimensional state feedback controller for the soft manipulator to track a desired position trajectory in the task space.

The control design relies on recognizing that the complete system \eqref{eq:task representation 2D} has a lower-triangular structure in the sense that the rotational motion $R$ can be viewed as a virtual input of the translational equation. 
Therefore, we adopt an inner-outer loop design philosophy. 
In the outer loop, we design desired rotational motions to rotate the translational motions in a direction that asymptotically achieves position tracking.
In the inner loop, we design inputs $m_c$ for the rotational equation to track their desired motions.
Both designs are inspired by the energy decay property of damped wave equations.
We will assume the states $\{p,R\}$ are available.
It is relatively easy to obtain estimates of $p$ using cameras.
One can then estimate $R$ using the extended Kalman filter for Cosserat-rod models reported in \cite{zheng2022pde}.

\noindent\textbf{Outer loop.} Define the following translational error term:
\begin{align*}
\tilde{p}=p-p^*.
\end{align*}
By the first equation of \eqref{eq:task representation 2D}, we obtain that $\tilde{p}$ satisfies:
\begin{align}\label{eq:linear closed loop}
    \tilde{p}_{tt} & =(RK_3R^Tp_s-RK_3\bar{q})_s+ge_3-p_{tt}^*.
\end{align}
We view $R$ as a virtual input to this system and use it to reshape the translational dynamics such that $\tilde{p}$ satisfies the following damped wave equation:
\begin{align*}
    \tilde{p}_{tt}=(K_q\tilde{p}_s)_s-K_v\tilde{p}_t-K_p\tilde{p},
\end{align*}
which is known to converge exponentially under suitable assumptions on the coefficients $\{K_q,K_v,K_p\}$ and the boundary condition.
This inspires us to design the desired rotational motion $R^*(s,t)$ (alternatively $R_*$), in the form of
\begin{align*}
    R^*=\begin{bmatrix}
        \cos\theta^* & -\sin\theta^* \\
        \sin\theta^* & \cos\theta^*
    \end{bmatrix},
\end{align*}
such that at every $t$, the following ODE holds
\begin{align}\label{eq:desired rotation differential form}
\begin{split}
    & \quad (R^*K_3R_*^Tp_s-R^*K_3\bar{q})_s+ge_3-p_{tt}^* \\
    & =(K_q\tilde{p}_s)_s-K_v\tilde{p}_t-K_p\tilde{p},
\end{split}
\end{align}
for some positive-definite matrix-valued functions $K_q(s),K_v(s,t),K_p(s)\in\mathbb{R}^{2\times2}$, with the boundary conditions $\theta^*(0,t)=0$ and $\theta_s^*(\ell,t)=\theta_s(\ell,t)$ to ensure that the desired rotational trajectory is trackable.
It is important that $K_v$ can be a function of $t$ which ensures that \eqref{eq:desired rotation differential form} has a solution.
We should point out that if $K_v$ is prescribed, \eqref{eq:desired rotation differential form} may not admit a solution for $R^*$.
The correct procedure is to treat both $R^*(s,t)$ and $K_v(s,t)$ as independent variables and solve for them simultaneously at every $t$ with the constraints that $R^*$ is a rotation matrix and $K_v>0$.

\noindent\textbf{Inner loop.} Define the following rotational error term:
\begin{align*}
    \tilde{\theta} = \theta-\theta^*.
\end{align*}
By the second equation of \eqref{eq:task representation 2D}, $\tilde{\theta}$ satisfies:
\begin{align}\label{eq:angular closed loop}
    \tilde{\theta}_{tt}=K_4\theta_{ss}-\frac{1}{\rho J}(m_c)_s+\hat{p}_sRK_5(R^Tp_s-\bar{q})-\theta_{tt}^*.
\end{align}
We can use the same idea to reshape the rotational dynamics into a damped wave equation.
This motivates us to design the input $m_c$ such that
\begin{align}\label{eq:m_c differential form}
\begin{split}
    & \quad K_4\theta_{ss}-\frac{1}{\rho J}(m_c)_s+\hat{p}_sRK_5(R^Tp_s-\bar{q})-\theta_{tt}^* \\
    & =(k_u\tilde{\theta}_s)_s-k_w\tilde{\theta}_t-k_\theta\tilde{\theta},
\end{split}
\end{align}
for some functions $k_u(s),k_w(s),k_\theta(s)>0$ with the boundary condition $m_c(\ell,t)=K_2\theta_s(\ell,t)$ according to \eqref{eq:BC}.
Equivalently, at every $t$, after substituting the boundary condition, $m_c$ can be computed by:
\begin{align}\label{eq:m_c}
\begin{split}
    m_c(s,t) & =\rho J[K_4\theta_s-k_u\tilde{\theta}_s](s,t) \\
    & \quad +\rho J\int_s^\ell\big[-k_w\tilde{\theta}_t-k_\theta\tilde{\theta} \\
    & \quad -\hat{p}_sRK_5(R^Tp_s-\bar{q})+\theta_{tt}^*\big](\tau,t)d\tau.
\end{split}
\end{align}

We can prove that the closed-loop system under such an inner-outer loop control is exponentially stable.

\begin{theorem}\label{thm:stability}
Consider \eqref{eq:task representation 2D}.
Let $R^*$ be computed by \eqref{eq:desired rotation differential form} and $m_c$ be given by \eqref{eq:m_c}.
Let the smallest eigenvalue of $K_q$ be sufficiently large such that $(K_q+RK_3R^T-R^*K_3R_*^T)$ is positive-definite for all $s$ at $t=0$.
Then as $t\to\infty$, the following convergence holds exponentially,
\begin{align*}
    & \big(\|\tilde{\theta}(\cdot,t)\|_{L^\infty},\|\tilde{\theta}_t(\cdot,t)\|_{L^\infty},\|\tilde{\theta}_s(\cdot,t)\|_{L^2}\big)\to0, \\
    & \big(\|\tilde{p}(\cdot,t)\|_{L^2},\|\tilde{p}_t(\cdot,t)\|_{L^2},\|\tilde{p}_s(\cdot,t)\|_{L^2}\big)\to0.
\end{align*}
\end{theorem}

\begin{proof}
The proof mainly consists of two arguments.
In the inner loop, $\{\tilde{\theta},\tilde{\theta}_t,\tilde{\theta}_s\}$ converge exponentially under \eqref{eq:m_c}.
In the outer loop, $\{\tilde{p},\tilde{p}_t,\tilde{p}_s\}$ become input-to-state stable \cite{dashkovskiy2013input} after at most a finite time (once $\{\tilde{\theta},\tilde{\theta}_t\}$ become small) and eventually converges exponentially.

(1) Inner loop.
We prove $(\|\tilde{\theta}\|_{L^\infty},\|\tilde{\theta}_t\|_{L^\infty},\|\tilde{\theta}_s\|_{L^2})\to0$.
Substituting \eqref{eq:m_c differential form} into \eqref{eq:angular closed loop}, we obtain
\begin{align}\label{eq:angular error system}
    \tilde{\theta}_{tt}=(k_u\tilde{\theta}_s)_s-k_w\tilde{\theta}_t-k_\theta\tilde{\theta},
\end{align}
with boundary conditions $\tilde{\theta}(0,t)=0$ and $\tilde{\theta}_s(\ell,t)=0$.
Consider a Lyapunov functional
\begin{align*}
    \mathcal{V}_1(t) & =\frac{1}{2}\int_0^\ell k_u\tilde{\theta}_s^2 + \tilde{\theta}_t^2 + 2c\tilde{\theta}\tilde{\theta}_t + k_\theta\tilde{\theta}^2 ds \\
    & =\frac{1}{2}\int_0^\ell k_u\tilde{\theta}_s^2 + \begin{bmatrix}
        \tilde{\theta}_t \\
        \tilde{\theta}
    \end{bmatrix}^T
    \begin{bmatrix}
        1 & c \\
        c & k_\theta
    \end{bmatrix}
    \begin{bmatrix}
        \tilde{\theta}_t \\
        \tilde{\theta}
    \end{bmatrix}
    ds,
\end{align*}
where $c>0$ is a constant to be determined later.
We have
\begin{align*}
    \frac{d}{dt}\mathcal{V}_1 & =\int_0^\ell k_u\tilde{\theta}_s\tilde{\theta}_{st} + \tilde{\theta}_t\tilde{\theta}_{tt} + c\tilde{\theta}\tilde{\theta}_{tt} + c\tilde{\theta}_t^2 + 
    k_\theta\tilde{\theta}\tilde{\theta}_t ds \\
    & =\int_0^\ell k_u\tilde{\theta}_s\tilde{\theta}_{st} + (\tilde{\theta}_t+c\tilde{\theta})[(k_u\tilde{\theta}_s)_s-k_w\tilde{\theta}_t-k_\theta\tilde{\theta}] \\
    & \quad+c\tilde{\theta}_t^2 + 
    k_\theta\tilde{\theta}\tilde{\theta}_t ds.
\end{align*}
Using integration by parts and the boundary condition,
\begin{align*} 
    \int_0^\ell (\tilde{\theta}_t+c\tilde{\theta})(k_u\tilde{\theta}_s)_s ds=\int_0^\ell -k_u(\tilde{\theta}_{ts}+c\tilde{\theta}_s)\tilde{\theta}_s ds.
\end{align*}
Then,
\begin{align*}
    \frac{d}{dt}\mathcal{V}_1 & =\int_0^\ell -ck_u\tilde{\theta}_s^2 -(k_w-c)\tilde{\theta}_t^2 - 
    ck_w\tilde{\theta}\tilde{\theta}_t - ck_\theta\tilde{\theta}^2 ds \\
    & =-\int_0^\ell ck_u\tilde{\theta}_s^2 + \begin{bmatrix}
        \tilde{\theta}_t \\
        \tilde{\theta}
    \end{bmatrix}^T
    \begin{bmatrix}
        k_w-c & \frac{ck_w}{2} \\
        \frac{ck_w}{2} & ck_\theta
    \end{bmatrix}
    \begin{bmatrix}
        \tilde{\theta}_t \\
        \tilde{\theta}
    \end{bmatrix}
    ds.
\end{align*}

We can choose $c$ such that
\begin{align}\label{eq:c condition}
    0<c<\inf_s\Big\{\sqrt{k_\theta(s)},\frac{k_\theta k_w}{k_\theta+k_w^2/4}(s)\Big\}.
\end{align}
Then $\mathcal{V}_1$ is positive-definite and $\frac{d}{dt}\mathcal{V}_1$ is negative-definite.
We obtain that $(\|\tilde{\theta}(\cdot,t)\|_{L^2},\|\tilde{\theta}_t(\cdot,t)\|_{L^2},\|\tilde{\theta}_s(\cdot,t)\|_{L^2})\to0$ exponentially.
Note that $(\|\tilde{\theta}(\cdot,t)\|_{L^2},\|\tilde{\theta}_s(\cdot,t)\|_{L^2})\to0$ implies $\|\tilde{\theta}(\cdot,t)\|_{L^\infty}\to0$.
Next, define $\eta=\tilde{\theta}_t$.
Since all the coefficients in \eqref{eq:angular error system} are independent of $t$, one can take the time derivative on both sides and find that $\eta$ satisfies the same equation as $\tilde{\theta}$.
By the same argument, we can prove that $(\|\eta(\cdot,t)\|_{L^2},\|\eta_t(\cdot,t)\|_{L^2},\|\eta_s(\cdot,t)\|_{L^2})\to0$ exponentially and hence $\|\tilde{\theta}_t(\cdot,t)\|_{L^\infty}=\|\eta(\cdot,t)\|_{L^\infty}\to0$ exponentially.

(2) Outer loop.
We prove $(\|\tilde{p}\|_{L^2},\|\tilde{p}_t\|_{L^2},\|\tilde{p}_s\|_{L^2})\to0$.
Substituting \eqref{eq:desired rotation differential form} into \eqref{eq:linear closed loop}, we obtain
\begin{align*}
    \tilde{p}_{tt} & =\big((K_q+\Phi)\tilde{p}_s\big)_s-K_v\tilde{p}_t-K_p\tilde{p}+\Psi,
\end{align*}
where
\begin{align*}
    \Phi & =RK_3R^T-R^*K_3R_*^T,\\
    \Psi & =(\Phi p_s^*)_s+(R^*-R)_sK_3\bar{q}.
\end{align*}
According to \eqref{eq:desired trajectory}, the boundary condition is given by $\tilde{p}(0,t)=0$ and $\tilde{p}_s(\ell,t)=0$.
Consider a Lyapunov functional
\begin{align*}
    \mathcal{V}_2(t)=\frac{1}{2}\int_0^\ell \tilde{p}_s^T(K_q+\Phi)\tilde{p}_s + \tilde{p}_t^T\tilde{p}_t + 2\tilde{p}^TC\tilde{p}_t + \tilde{p}^TK_p\tilde{p} ds,
\end{align*}
where $C>0$ is a positive-definite constant matrix such that
\begin{align*}
    K_p^{\frac{1}{2}}-C \text{ and } K_v-C-\frac{K_vCK_p^{-1}K_v}{4}
\end{align*}
are both positive-definite for all $s$, which is essentially the matrix version of condition \eqref{eq:c condition}.
(The existence of such an $C$ can be easily verified by assuming $C=c_1I$ for some constant $c_1>0$.)
By assumption, $\mathcal{V}_2$ is positive-definite at $t=0$.
Note that $\|\tilde{\theta}(\cdot,t)\|_{L^\infty}\to0$ exponentially implies that $\|\Phi(\cdot,t)\|_{L^\infty}\to0$ exponentially.
Hence, $\mathcal{V}_2$ is positive-definite for all $t\geq0$.
Next, by taking the time derivative and using integration by parts in a similar way as for $\frac{d}{dt}\mathcal{V}_1$, we obtain
\begin{align*}
    \frac{d}{dt}\mathcal{V}_2 & =\int_0^\ell \tilde{p}_s^T(K_q+\Phi)\tilde{p}_{st} + \frac{1}{2}\tilde{p}_s^T\Phi_t\tilde{p}_s + \tilde{p}_t^T\tilde{p}_{tt} \\
    & \quad + \tilde{p}^TC\tilde{p}_{tt} + \tilde{p}_t^TC\tilde{p}_t + \tilde{p}^TK_p\tilde{p}_t ds \\
    & =\int_0^\ell -\tilde{p}_s^T(CK_q+C\Phi-\frac{1}{2}\Phi_t)\tilde{p}_s + (\tilde{p}_t+C\tilde{p})^T\Psi \\
    & \quad -\begin{bmatrix}
        \tilde{p}_t \\
        \tilde{p}
    \end{bmatrix}^T
    \begin{bmatrix}
        K_v-C & \frac{1}{2}K_vC \\
        \frac{1}{2}CK_v & CK_p
    \end{bmatrix}
    \begin{bmatrix}
        \tilde{p}_t \\
        \tilde{p}
    \end{bmatrix} ds,
\end{align*}
where the last term is negative-definite by our selection of $C$.
Note that $(CK_q+C\Phi-\frac{1}{2}\Phi_t)$ is not necessarily positive-definite.
However, $(\|\tilde{\theta}(\cdot,t)\|_{L^\infty},\|\tilde{\theta}_t(\cdot,t)\|_{L^\infty})\to0$ exponentially implies that $\|\Phi_t(\cdot,t)\|_{L^\infty}\to0$ exponentially.
This means that there exists a finite time $T>0$ such that when $t>T$, $(CK_q+C\Phi-\frac{1}{2}\Phi_t)$ becomes positive-definite for all $s$ and $\mathcal{V}_2$ becomes input-to-state stable \cite{dashkovskiy2013input} with respect to $\|\Psi(\cdot,t)\|_{L^2}$.
Since $(\|\tilde{\theta}(\cdot,t)\|_{L^\infty},\|\tilde{\theta}_s(\cdot,t)\|_{L^2})\to0$ exponentially implies that $\|\Psi(\cdot,t)\|_{L^2}\to0$ exponentially, we conclude that $(\|\tilde{p}(\cdot,t)\|_{L^2},\|\tilde{p}_t(\cdot,t)\|_{L^2},\|\tilde{p}_s(\cdot,t)\|_{L^2})\to0$ exponentially.
\end{proof}

\begin{remark}
Our control design takes place in the task space so we completely avoid the inverse kinematics problem encountered if using PCC models \cite{webster2010design}. 
The control algorithm in either loop may be replaced by alternative algorithms as long as they have suitable stabilizing properties.
For example, the inner loop is exponentially stable and the outer loop is input-to-state stable (at least after a finite time).
\end{remark}

The control input given by \eqref{eq:m_c} is infinite-dimensional.
In the implementation, it needs to be approximated by a finite number of actuators.
Take fluidic actuators as an example. 
Then one needs to find values of the air pressure $\{P_i\}_{i=1}^\alpha$ such that the generated actuation $m_c$ (finite-dimensional) according to \eqref{eq:actuator model} approximates the designed $m_c$ (infinite-dimensional) in \eqref{eq:m_c}.
Intuitively, we can achieve better approximations with more actuators and suitable placements. 
On the other hand, if a position trajectory is not achievable based on the current actuator allocation, then it is impossible to accurately approximate the designed $m_c$.
This allocation problem highly depends on the actuators and is under study.

\section{Simulation study}
\label{section:simulation}
\begin{figure*}[t]
    \centering
    \begin{subfigure}[b]{0.22\textwidth}
        \centering
        \includegraphics[width=\textwidth]{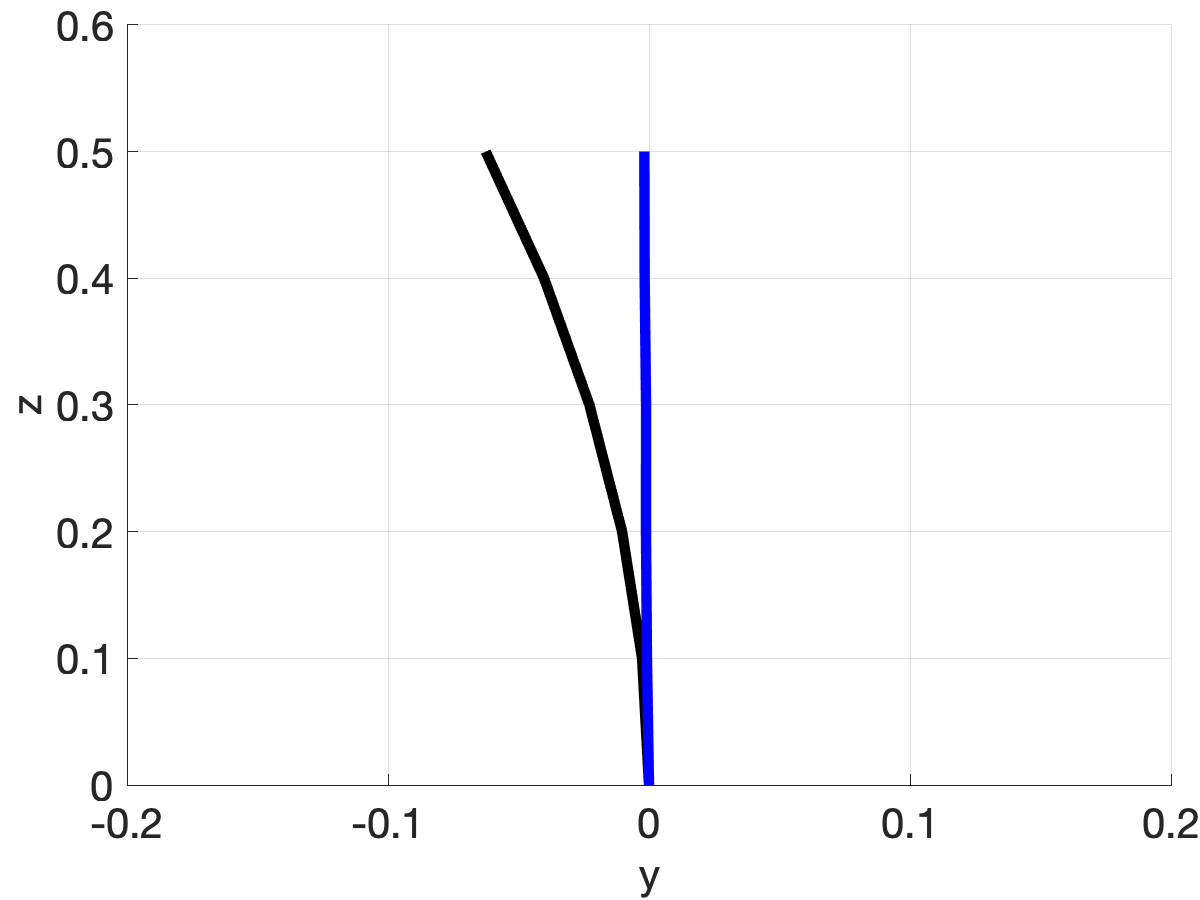}
    \end{subfigure}
    \begin{subfigure}[b]{0.22\textwidth}
        \centering
        \includegraphics[width=\textwidth]{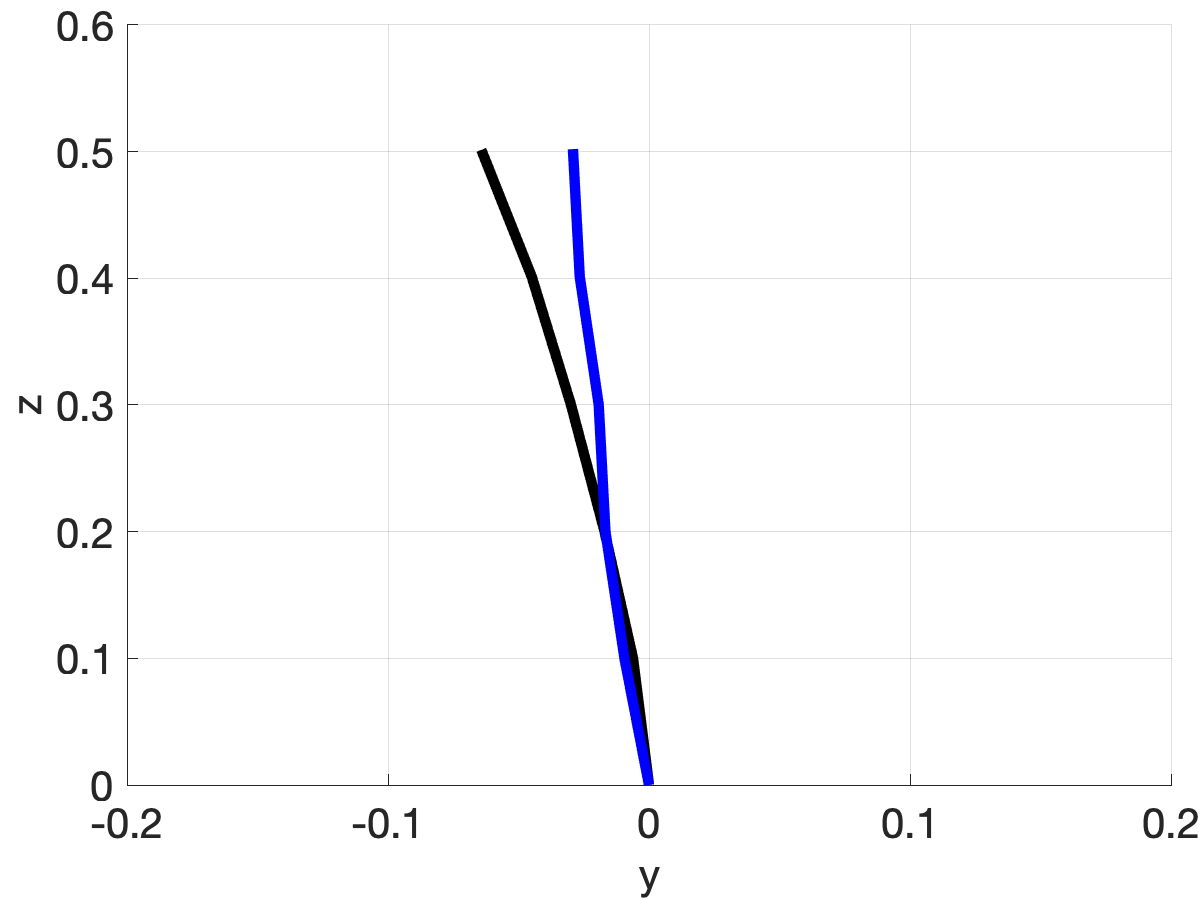}
    \end{subfigure}
    \begin{subfigure}[b]{0.22\textwidth}
        \centering
        \includegraphics[width=\textwidth]{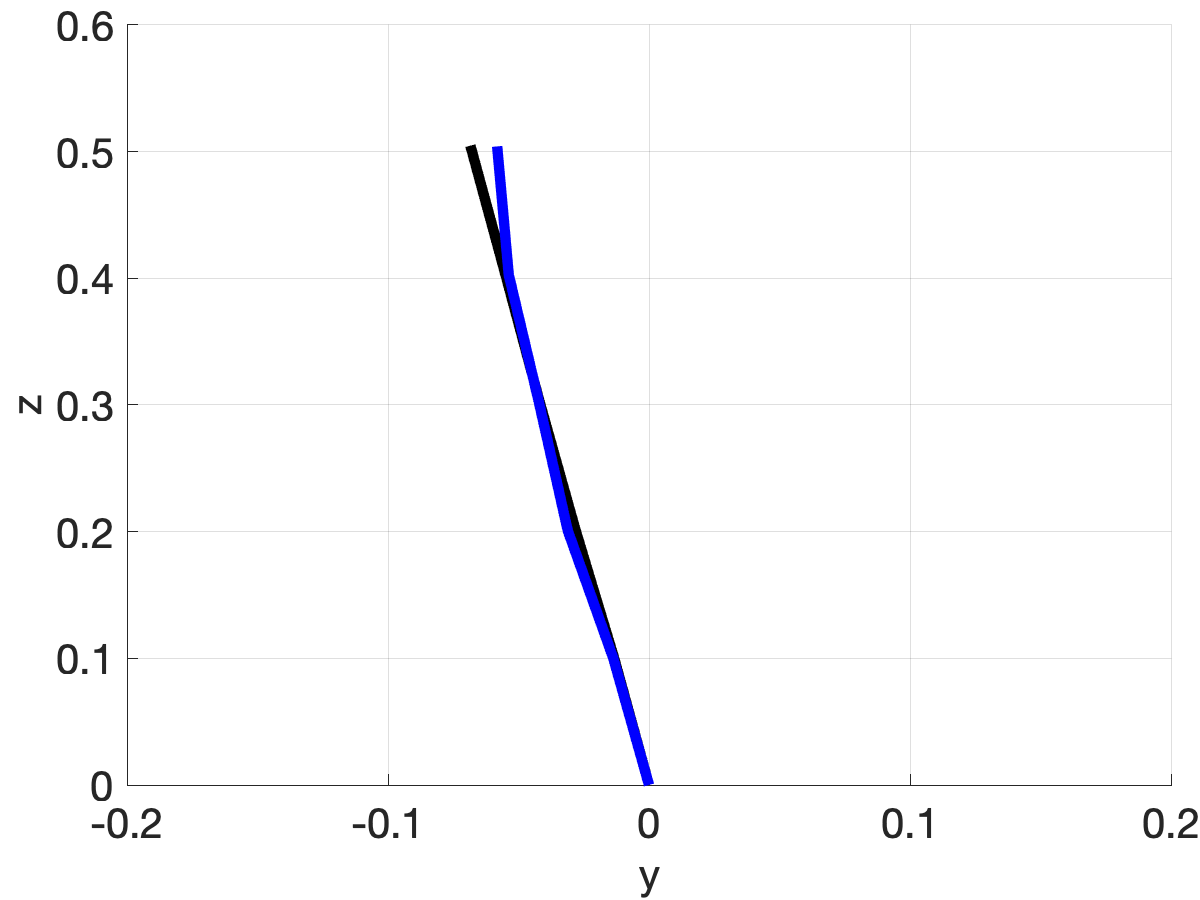}
    \end{subfigure}
    \begin{subfigure}[b]{0.22\textwidth}
        \centering
        \includegraphics[width=\textwidth]{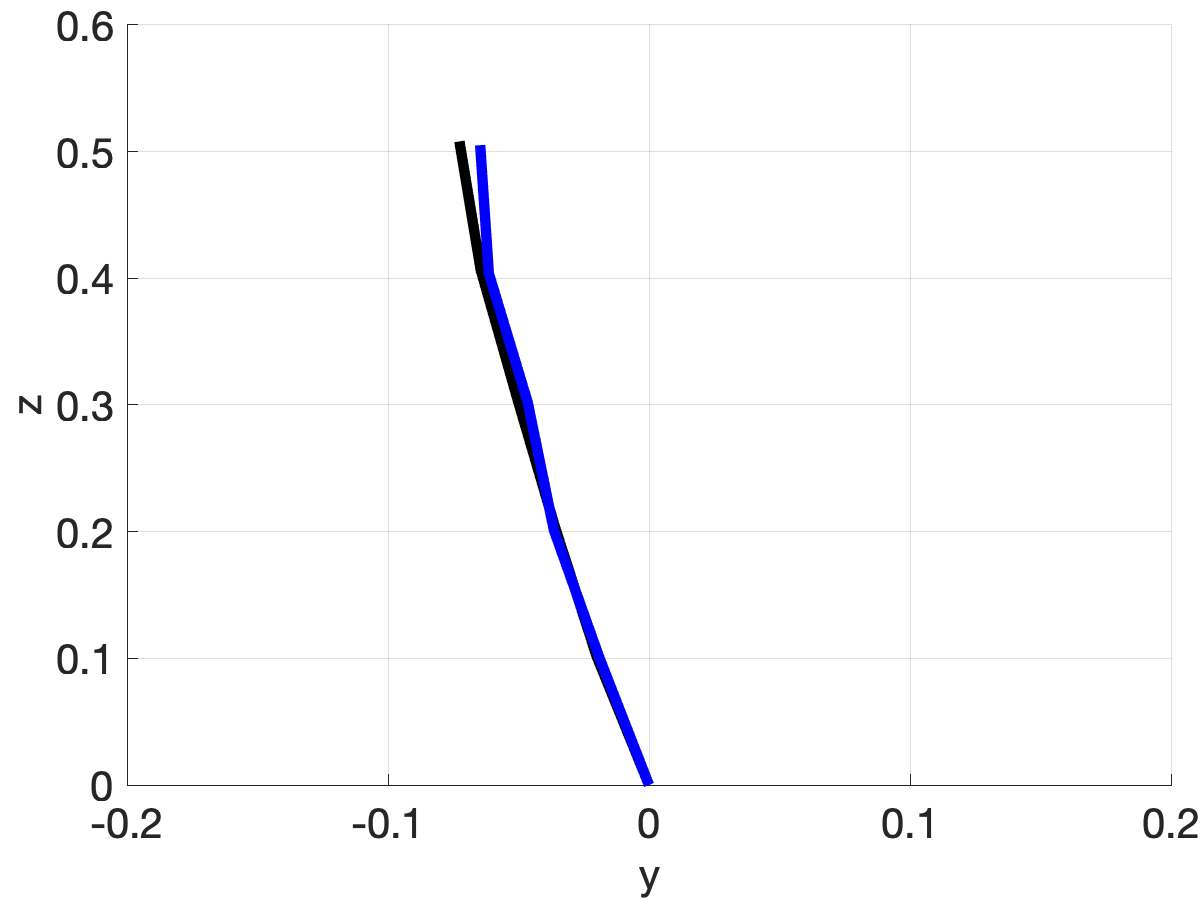}
    \end{subfigure}
    \caption{The position tracking process (from left to right). Black: desired position. Blue: actual position.}
    \label{fig:tracking}
\end{figure*}

A simulation study is performed on MATLAB to verify the proposed control algorithms.

\textit{Setup.} The system parameters for \eqref{eq:task representation 2D} are chosen as $\ell=0.5$, $K_3=\operatorname{diag}[1,1.5]$, $K_4=1$, and $K_5=1.5$.
The system \eqref{eq:task representation 2D} is simulated using finite differences where we set $ds=0.05$ and $dt=0.005$.
In other words, we use $N=11$ points to represent a rod.
(Note that these system parameters are chosen to be small values to avoid numerical instability in the finite difference method and may not reflect a real soft manipulator.
A more delicate simulator for soft manipulators is under development.)
The manipulator is initially undeformed and lies on the $z$-axis.

\textit{Algorithm implementation.} In the outer loop, we set the control gains to be $K_q=1$ and $K_p=4$.
To solve for the desired rotation matrix $R^*$ using \eqref{eq:desired rotation differential form}, we notice that at every $t$, \eqref{eq:desired rotation differential form} can be written as an ODE of $s$ in the following form:
\begin{align*}
    \frac{d}{ds}(R^*K_3R_*^Ta-R^*b)+K_vc+d=0,
\end{align*}
where $a,b,c,d:[0,\ell]\to\mathbb{R}^2$ are given vector fields.
We let $K_v=\operatorname{diag}[K_v^1,K_v^2]$ be a diagonal matrix.
After spatial discretization, each component of a vector field becomes an $N$-dimensional vector, and the differentiation operator $\frac{d}{ds}$ is represented by a matrix $A\in\mathbb{R}^{N\times N}$.
We thus obtain the following $N$-dimensional algebraic equation for each component $i$:
\begin{align*}
    A(R^*K_3R_*^Ta-R^*b)^i+K_v^ic^i+d^i=0,\quad i=1,2,
\end{align*}
subject to the constraints that $\theta_*^j\in[-\pi,\pi)$ and $K_v^{1,j},K_v^{2,j}>0$, $j=1,\dots,N$, which can then be solved using nonlinear optimization algorithms such as ``lsqnonlin'' in MATLAB.
Note that since the deformation is continuous, $\{\theta_*^j,K_v^{1,j},K_v^{2,j}\}$ are supposed to change continuously in time.
Thus, using the computed solution from the last time step to start the search of the current time step significantly improves the efficiency of the optimization algorithm.
In the inner loop, the control gains are chosen to be $k_u=0.5$, $k_\theta=4$ and $k_w=2$.

\textit{Result.} The position tracking task is illustrated in Fig. \ref{fig:tracking}.
The desired position is initially bent and the manipulator is able to converge to the desired position.
The $L^2$ norms of the position and velocity errors are given in Fig. \ref{fig:error}

\begin{figure}[hbt!]
    \centering
    \includegraphics[width=0.9\columnwidth]{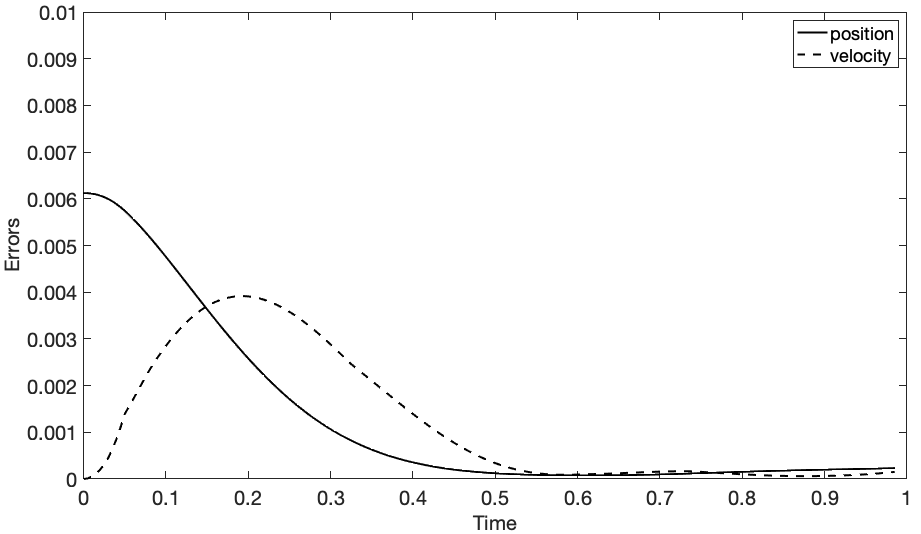}
    \caption{Tracking errors. Solid: $\|\tilde{p}\|_{L^2}$. Dashed: $\|\tilde{p}_t\|_{L^2}$}
    \label{fig:error}
\end{figure}

\section{Conclusion}
\label{section:conclusion}
In this work, we explored infinite-dimensional control theory for soft manipulators that are modeled by nonlinear Cosserat-rod PDEs.
We presented an inner-outer loop control design for position tracking in planar task space.
We first designed desired rotational motions that asymptotically achieve position tracking and then designed inputs for the rotational component to track their desired motions.
Both designs were inspired by the energy decay property of damped wave equations.
Exponential stability was proved.
These results suggested the promising role of infinite-dimensional control theory for soft robots.
Our future work is to extend the inner-outer loop design to the 3D case and test the control algorithms on a real platform.

\bibliographystyle{IEEEtran}
\bibliography{References}

\begin{thebibliography}{10}
\providecommand{\url}[1]{#1}
\csname url@samestyle\endcsname
\providecommand{\newblock}{\relax}
\providecommand{\bibinfo}[2]{#2}
\providecommand{\BIBentrySTDinterwordspacing}{\spaceskip=0pt\relax}
\providecommand{\BIBentryALTinterwordstretchfactor}{4}
\providecommand{\BIBentryALTinterwordspacing}{\spaceskip=\fontdimen2\font plus
\BIBentryALTinterwordstretchfactor\fontdimen3\font minus
  \fontdimen4\font\relax}
\providecommand{\BIBforeignlanguage}[2]{{%
\expandafter\ifx\csname l@#1\endcsname\relax
\typeout{** WARNING: IEEEtran.bst: No hyphenation pattern has been}%
\typeout{** loaded for the language `#1'. Using the pattern for}%
\typeout{** the default language instead.}%
\else
\language=\csname l@#1\endcsname
\fi
#2}}
\providecommand{\BIBdecl}{\relax}
\BIBdecl

\bibitem{rus2015design}
D.~Rus and M.~T. Tolley, ``Design, fabrication and control of soft robots,''
  \emph{Nature}, vol. 521, no. 7553, pp. 467--475, 2015.

\bibitem{burgner2015continuum}
J.~Burgner-Kahrs, D.~C. Rucker, and H.~Choset, ``Continuum robots for medical
  applications: A survey,'' \emph{IEEE Transactions on Robotics}, vol.~31,
  no.~6, pp. 1261--1280, 2015.

\bibitem{marchese2014autonomous}
A.~D. Marchese, C.~D. Onal, and D.~Rus, ``Autonomous soft robotic fish capable
  of escape maneuvers using fluidic elastomer actuators,'' \emph{Soft
  robotics}, vol.~1, no.~1, pp. 75--87, 2014.

\bibitem{della2021model}
C.~Della~Santina, C.~Duriez, and D.~Rus, ``Model based control of soft robots:
  A survey of the state of the art and open challenges,'' \emph{arXiv preprint
  arXiv:2110.01358}, 2021.

\bibitem{webster2010design}
R.~J. Webster~III and B.~A. Jones, ``Design and kinematic modeling of constant
  curvature continuum robots: A review,'' \emph{The International Journal of
  Robotics Research}, vol.~29, no.~13, pp. 1661--1683, 2010.

\bibitem{polygerinos2015modeling}
P.~Polygerinos, Z.~Wang, J.~T. Overvelde, K.~C. Galloway, R.~J. Wood,
  K.~Bertoldi, and C.~J. Walsh, ``Modeling of soft fiber-reinforced bending
  actuators,'' \emph{IEEE Transactions on Robotics}, vol.~31, no.~3, pp.
  778--789, 2015.

\bibitem{marchese2016dynamics}
A.~D. Marchese, R.~Tedrake, and D.~Rus, ``Dynamics and trajectory optimization
  for a soft spatial fluidic elastomer manipulator,'' \emph{The International
  Journal of Robotics Research}, vol.~35, no.~8, pp. 1000--1019, 2016.

\bibitem{della2020model}
C.~Della~Santina, R.~K. Katzschmann, A.~Bicchi, and D.~Rus, ``Model-based
  dynamic feedback control of a planar soft robot: trajectory tracking and
  interaction with the environment,'' \emph{The International Journal of
  Robotics Research}, vol.~39, no.~4, pp. 490--513, 2020.

\bibitem{franco2021energy}
E.~Franco and A.~Garriga-Casanovas, ``Energy-shaping control of soft continuum
  manipulators with in-plane disturbances,'' \emph{The International Journal of
  Robotics Research}, vol.~40, no.~1, pp. 236--255, 2021.

\bibitem{duriez2013control}
C.~Duriez, ``Control of elastic soft robots based on real-time finite element
  method,'' in \emph{2013 IEEE international conference on robotics and
  automation}.\hskip 1em plus 0.5em minus 0.4em\relax IEEE, 2013, pp.
  3982--3987.

\bibitem{goury2018fast}
O.~Goury and C.~Duriez, ``Fast, generic, and reliable control and simulation of
  soft robots using model order reduction,'' \emph{IEEE Transactions on
  Robotics}, vol.~34, no.~6, pp. 1565--1576, 2018.

\bibitem{simo1988dynamics}
J.~C. Simo and L.~Vu-Quoc, ``On the dynamics in space of rods undergoing large
  motions—a geometrically exact approach,'' \emph{Computer methods in applied
  mechanics and engineering}, vol.~66, no.~2, pp. 125--161, 1988.

\bibitem{antman2005nonlinear}
S.~Antman, \emph{Nonlinear problems of elasticity}.\hskip 1em plus 0.5em minus
  0.4em\relax Springer Science \& Business Media, 2005, vol. 107.

\bibitem{rucker2011statics}
D.~C. Rucker and R.~J. Webster~III, ``Statics and dynamics of continuum robots
  with general tendon routing and external loading,'' \emph{IEEE Transactions
  on Robotics}, vol.~27, no.~6, pp. 1033--1044, 2011.

\bibitem{renda2014dynamic}
F.~Renda, M.~Giorelli, M.~Calisti, M.~Cianchetti, and C.~Laschi, ``Dynamic
  model of a multibending soft robot arm driven by cables,'' \emph{IEEE
  Transactions on Robotics}, vol.~30, no.~5, pp. 1109--1122, 2014.

\bibitem{till2019real}
J.~Till, V.~Aloi, and C.~Rucker, ``Real-time dynamics of soft and continuum
  robots based on cosserat rod models,'' \emph{The International Journal of
  Robotics Research}, vol.~38, no.~6, pp. 723--746, 2019.

\bibitem{renda2020geometric}
F.~Renda, C.~Armanini, V.~Lebastard, F.~Candelier, and F.~Boyer, ``A geometric
  variable-strain approach for static modeling of soft manipulators with tendon
  and fluidic actuation,'' \emph{IEEE Robotics and Automation Letters}, vol.~5,
  no.~3, pp. 4006--4013, 2020.

\bibitem{janabi2021cosserat}
F.~Janabi-Sharifi, A.~Jalali, and I.~D. Walker, ``Cosserat rod-based dynamic
  modeling of tendon-driven continuum robots: A tutorial,'' \emph{IEEE Access},
  vol.~9, pp. 68\,703--68\,719, 2021.

\bibitem{grazioso2019geometrically}
S.~Grazioso, G.~Di~Gironimo, and B.~Siciliano, ``A geometrically exact model
  for soft continuum robots: The finite element deformation space
  formulation,'' \emph{Soft robotics}, vol.~6, no.~6, pp. 790--811, 2019.

\bibitem{george2020first}
T.~George~Thuruthel, F.~Renda, and F.~Iida, ``First-order dynamic modeling and
  control of soft robots,'' \emph{Frontiers in Robotics and AI}, vol.~7, p.~95,
  2020.

\bibitem{doroudchi2021configuration}
A.~Doroudchi and S.~Berman, ``Configuration tracking for soft continuum robotic
  arms using inverse dynamic control of a cosserat rod model,'' in \emph{2021
  IEEE 4th International Conference on Soft Robotics (RoboSoft)}.\hskip 1em
  plus 0.5em minus 0.4em\relax IEEE, 2021, pp. 207--214.

\bibitem{chang2020energy}
H.-S. Chang, U.~Halder, C.-H. Shih, A.~Tekinalp, T.~Parthasarathy, E.~Gribkova,
  G.~Chowdhary, R.~Gillette, M.~Gazzola, and P.~G. Mehta, ``Energy shaping
  control of a cyberoctopus soft arm,'' in \emph{2020 59th IEEE Conference on
  Decision and Control (CDC)}.\hskip 1em plus 0.5em minus 0.4em\relax IEEE,
  2020, pp. 3913--3920.

\bibitem{zheng2022pde}
T.~Zheng and H.~Lin, ``Pde-based dynamic control and estimation of soft robotic
  arms,'' in \emph{2022 61st IEEE Conference on Decision and Control
  (CDC)}.\hskip 1em plus 0.5em minus 0.4em\relax IEEE, 2022.

\bibitem{rodriguez2020boundary}
C.~Rodriguez and G.~Leugering, ``Boundary feedback stabilization for the
  intrinsic geometrically exact beam model,'' \emph{SIAM Journal on Control and
  Optimization}, vol.~58, no.~6, pp. 3533--3558, 2020.

\bibitem{li2001semi}
T.-T. Li and Y.~Jin, ``Semi-global c1 solution to the mixed initial-boundary
  value problem for quasilinear hyperbolic systems,'' \emph{Chinese Annals of
  Mathematics}, vol.~22, no.~03, pp. 325--336, 2001.

\bibitem{dashkovskiy2013input}
S.~Dashkovskiy and A.~Mironchenko, ``Input-to-state stability of
  infinite-dimensional control systems,'' \emph{Mathematics of Control,
  Signals, and Systems}, vol.~25, no.~1, pp. 1--35, 2013.

\end{thebibliography}

\end{document}